\documentclass[letterpaper]{article} 
\usepackage{aaai24}  
\usepackage{times}  
\usepackage{helvet}  
\usepackage{courier}  
\usepackage[hyphens]{url}  
\usepackage{graphicx} 
\usepackage{natbib}  
\usepackage{caption} 
\usepackage{algorithm}
\usepackage{algorithmic}
\usepackage{multirow}
\usepackage[table,xcdraw]{xcolor}
\usepackage{amsmath}
\usepackage{array}
\usepackage{float}

\usepackage{amssymb}
\usepackage{subcaption}

\usepackage{newfloat}
\usepackage{listings}
\DeclareCaptionStyle{ruled}{labelfont=normalfont,labelsep=colon,strut=off} 
\floatstyle{ruled}
\newfloat{listing}{tb}{lst}{}
\floatname{listing}{Listing}
\title{Unraveling Batch Normalization for Realistic Test-Time Adaptation} 
\author{
    Zixian Su\textsuperscript{\rm 1, \rm 2},
    Jingwei Guo\textsuperscript{\rm 1, \rm 2},
    Kai Yao\textsuperscript{\rm 1, \rm 2},
        Xi Yang\textsuperscript{\rm 1}\thanks{Corresponding author.},
            Qiufeng Wang\textsuperscript{\rm 1},
        Kaizhu Huang\textsuperscript{\rm 3}\footnotemark[1]
}
\affiliations{
    \textsuperscript{\rm 1}Department of Intelligent Science, Xi'an Jiaotong-Liverpool University, Suzhou, China\\
        \textsuperscript{\rm 2}Faculty of Science and Engineering, University of Liverpool, Liverpool, the United Kingdom\\
    \textsuperscript{\rm 3}Data Science Research Center, Duke Kunshan University, Kunshan, China\\
    zixian.su@liverpool.ac.uk, jingwei.guo@liverpool.ac.uk\\
        xi.yang01@xjtlu.edu.cn,
    kaizhu.huang@dukekunshan.edu.cn

}

\usepackage{amsthm}
\newtheorem{proposition}{Proposition}

\begin{document}
\maketitle

\begin{abstract}
While recent test-time adaptations exhibit efficacy by adjusting batch normalization to narrow domain disparities, their effectiveness diminishes with realistic mini-batches due to inaccurate target estimation. As previous attempts merely introduce source statistics to mitigate this issue, the fundamental problem of inaccurate target estimation still persists, leaving the intrinsic test-time domain shifts unresolved. This paper delves into the problem of mini-batch degradation. By unraveling batch normalization, we discover that the inexact target statistics largely stem from the substantially reduced class diversity in batch. Drawing upon this insight, we introduce a straightforward tool, Test-time Exponential Moving Average (TEMA), to bridge the class diversity gap between training and testing batches. Importantly, our TEMA adaptively extends the scope of typical methods beyond the current batch to incorporate a diverse set of class information, which in turn boosts an accurate target estimation. Built upon this foundation, we further design a novel layer-wise rectification strategy to consistently promote test-time performance. Our proposed method enjoys a unique advantage as it requires neither training nor tuning parameters, offering a truly hassle-free solution. It significantly enhances model robustness against shifted domains and maintains resilience in diverse real-world scenarios with various  batch sizes, achieving state-of-the-art performance on several major benchmarks. Code is available at \url{https://github.com/kiwi12138/RealisticTTA}.
\end{abstract}

\section{Introduction}

Confronted with unseen environments, the effectiveness of deep neural networks~\cite{krizhevsky2017imagenet,he2016deep,deeplab} often suffers a decline due to domain shift~\cite{domainshift,domainshift2} --- an incongruity between the training (source) and testing (target) domains. To address this, Test-Time Adaptation (TTA)~\cite{wang2021tent,TTT} serves as a practical paradigm that enables pre-trained models to dynamically adapt with test streams.
Recent TTA research mainly resolves around exploring batch normalization (BN)~\cite{ioffe2015batch}. A crucial reason for this focus lies in the internal relation between normalization statistics and domain characteristics. Specifically, BN uses source statistics for normalization during training; however, applying the same parameters to differing target domains can lead to erratic results. Instead, normalizing with the estimated statistics from the target domain promotes appropriate standardization and superior generalization across diverse environments.

\begin{figure}[t]
  \centering
  \includegraphics[clip,trim=35.5cm 5cm 0cm 25.5cm,width=0.75\textwidth]
  {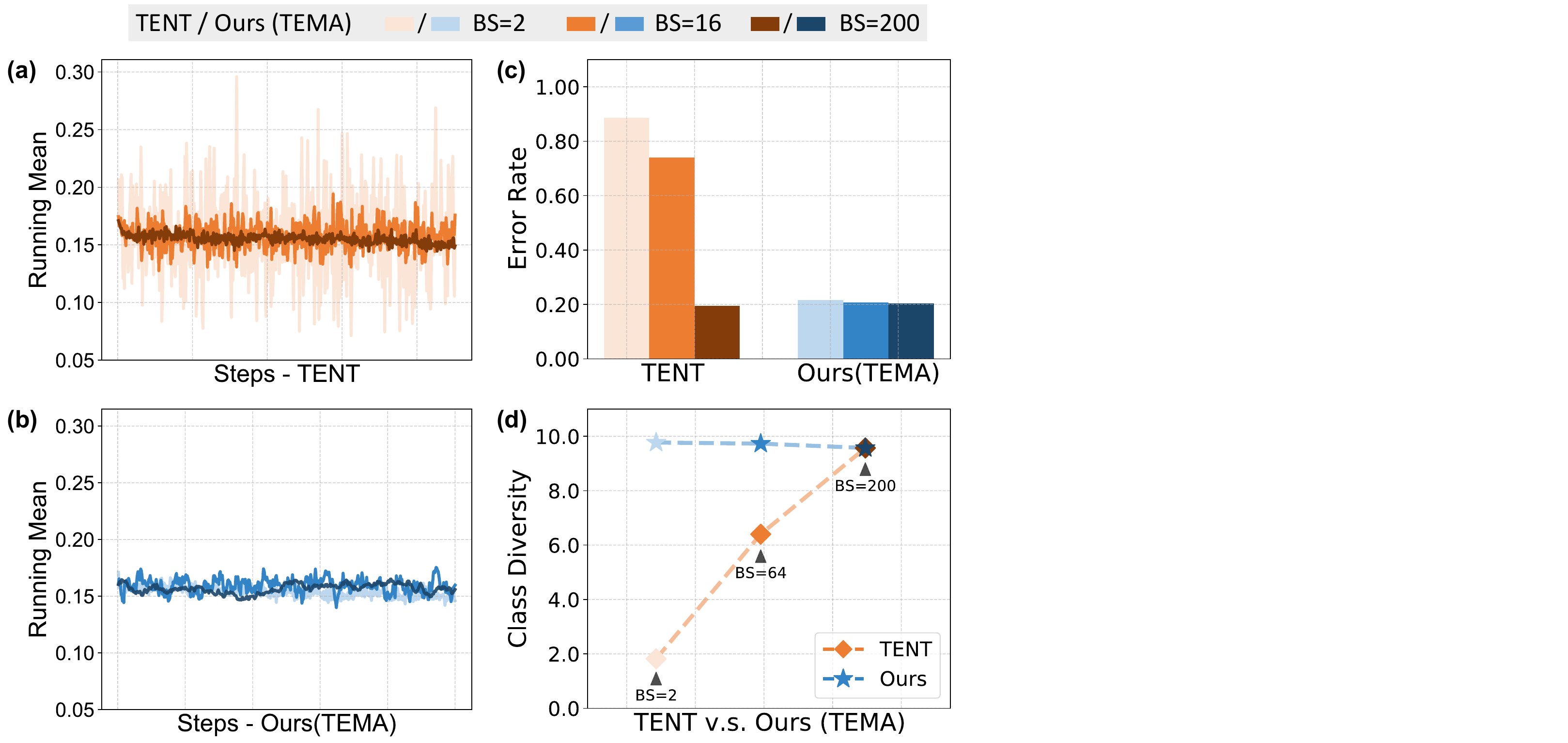}
\caption{(a)(b) Running mean statistics of one specific channel in the last BN layer during inference. (c) Performance comparison between TENT~\cite{wang2021tent} and Ours under different batch sizes.  (d) Quantified class diversity of TENT v.s. Ours under different batch sizes. Same color denotes same evaluation setting. As can be seen, class diversity plays a vital role in test-time performance.}
  \label{fig:banner}
\end{figure}

Despite their excellent performance under chosen settings, existing TTA methods face essential challenges when deployed in realistic settings~\cite{wang2022continual,lim2023ttn}. Most of them necessitate extensive target data, commonly  large batches. While in practice, mini-batches are often preferred, which may produce erroneous statistics and cause model degradation (see Figure~\ref{fig:banner}~(c)).
Although some attempts~\cite{schneider2020adaptBN,you2021test} have been made to alleviate this limitation by incorporating source statistics, they merely cover up the inaccurate target information, leaving the persistent domain shifts unresolved. These observations naturally prompt us to ask: 
\textit{What exactly causes the inaccurate target statistics? How can we mitigate these inaccuracies to alleviate domain shift? 
}

This paper explores these concerns by unraveling batch normalization.
A common belief in TTA research is that the accuracy of target estimation largely hinges on the quantity of batch samples~\cite{zhao2022delta,DUA,khurana2021sita}. On initial observation, it could be inferred that mini-batches, with fewer samples, would naturally produce flawed target statistics. However, this viewpoint is reductive.  As depicted in Figure~\ref{fig:banner}~(b)(c), the inaccurate estimation and performance degradation is efficiently reduced by increasing the class diversity  within mini-batches.
Our investigation indicates that \textit{the quality of target estimation is not only influenced by the sample count but, more fundamentally, by the diversity within those samples.} 
Distortions arising from this discrepancy, which is independent of domain traits, could further obscure the true domain shifts and compromise model performance.

To tackle the discrepancies between training and testing phases in class diversity, we introduce a straightforward remedy, termed Test-time Exponential Moving Average (TEMA). 
TEMA utilises past statistics from previous batches, using a weighted average to enrich the diversity of class information in the current batch's statistics. However, as the data scope of TEMA expands (anticipated increase in class diversity), the model starts to lose its precise grasp on the current batch's data, leading to unstable and disrupted normalization. This gives rise to a trade-off: 
while concrete local context may suppress information richness, a focus on the global aspect compromises information timeliness. Recognizing this dilemma, we devise a versatile strategy that dynamically tailors the momentum, a parameter regulating TEMA’s scope, to strike a balance between information richness and timeliness. 
This design broadens the scope of standard TTA methods beyond the current batch. By properly incorporating a diverse set of class information, TEMA stands out as a powerful tool for accurate target estimation irrespective of batch size at test time.
 
Built on the improvements from TEMA, one may directly replace source statistics with the target, which would risk model instability due to ever-changing parameters in BN layers, especially for complex tasks (see Table~\ref{tab:abla}). While source statistics may not be ideally suited for challenges in target domains, they are instrumental in maintaining the model's stability and robustness --- attributes often lacking in target data.
In response to this, we propose a novel layer-wise rectification strategy that compensates target statistics with the source. Specifically, we leverage the divergence between target and source distributions as a guiding metric to adjust their contributions to our final normalization statistics. A significant divergence mandates a heavier reliance on the source to stabilize model performance, whereas a smaller divergence calls for a greater emphasis on the target to pinpoint the domain shifts. 
Our method effectively balances model stability with the use of BN to accurately capture domain shifts, enhancing overall test-time performance. 
Notably, our strategy stands out from traditional methods as it eliminates the need for manual parameter tuning, making it far more practical in real-world scenarios.

The contributions are summarized as follows:
\begin{itemize}
    \item We investigate the underlying cause of inaccurate statistics during test-time adaptation, pinpointing the issue to the limited diversity of classes within batches.
    \item We introduce the Test-time Exponential Moving Average  with an  adaptive momentum mechanism. This approach dynamically balances the diversity of class information while ensuring timely updates. 
    \item We propose a novel layer-wise normalization rectification strategy, considering the distribution divergence, to promote overall test-time performance.
    \item Extensive experiments exhibit consistent improvement and demonstrate remarkable stability.
\end{itemize}

\section{Related Work}
\subsection{Test-time Adaption}
A common challenge deep neural networks face is a decrease in performance when training and testing data present divergent distributions. To address this issue, substantial efforts~\cite{you2019universal,partial,hoffman2018cycada} have been directed towards bridging the performance gap.
Recently, Test-time adaptation (TTA) has emerged as a solution to combat the distribution shift from source to target domains during testing~\cite{wang2021tent,TTT,T3A,shot}. 
Predominantly, these methods fall into two main categories: Test-time Training (TTT)~\cite{TTT,liu2021ttt++,TTT_mae} and Fully Test-Time Adaptation (FTTA)~\cite{wang2021tent,niu2022towards,MEMO}. TTT involves updating model parameters during testing and necessitates a specific auxiliary task during training, while FTTA presents a more challenging and realistic task by requiring the model to adapt online to the test stream without any modifications during training. This approach calls for the model to execute swift, real-time adjustments to effectively interpret and respond to the incoming data stream. In this paper, we direct our focus towards FTTA over TTT, as we are motivated by its capacity to meet the practical needs of dynamic and continuously evolving data environments.

\subsection{Batch Normalization in Test-time Adaptation}
Batch normalization (BN)~\cite{ioffe2015batch} has been widely used in deep neural networks for stable training and fast convergence. Recent TTAs mainly center on exploring the connections between batch normalization statistics and domain characteristics. One approach attempts to optimize the model during testing using target batch statistics and a specific loss function ~\cite{wang2021tent,niu2022efficient,niu2022towards}. Yet, this method neglects the challenge of accurately estimating the target distribution as the test batch size decreases, resulting in model deterioration under such circumstances. Another stream considers training-free methods by calibrating the normalization statistics. Since \citet{nado2020evaluating} suggested prediction-time BN, which purely uses target statistics for standardization, \citet{schneider2020adaptBN}  and \citet{you2021test} proposed to modify BN statistics by mixing the source and target via a predefined hyper-parameter to mitigate the intermediate covariate shift. The latter two realized the drawbacks of target statistics and tried to cover up the inaccuracies with the source statistics, while leaving the persistent problem unresolved.

It is noted that recent work, TTN~\cite{lim2023ttn}, bears some resemblance to ours, as both aim to address this inaccuracy issue. However, our approaches differ significantly in the solutions they employ. TTN augments source data to simulate test-time domain shifts and requires post-training to set a fixed balance between source and target statistics in final normalization. This strategy, while effective under certain conditions, demands extra training and lacks adaptability to varying target batches.
In contrast, our method directly captures target statistics reflecting the intrinsic domain shifts, eliminating extra training. During inference, we adaptively balance between source and target weights for each incoming batch based on inter-domain divergence. This design enhances  practicality in complex and diverse environments.

\subsection{Test-time Exponential Moving Average}
In previous studies, the Exponential Moving Average (EMA) mechanism was solely utilized during training to record more generalized source statistics with global information in BN layers.
Recently, \citet{zhao2022delta} and \citet{yuan2023robust} introduced  EMA during the test phase to stabilize target normalization statistics, which are particularly prone to inaccuracies under mini-batch conditions. This process was coined as TEMA. Despite their improvements, the authors neither analyzed the root cause of the inaccurate target statistics nor clarified why TEMA was effective at mitigating this issue. In stark contrast, our in-depth analysis reveals the essence of this problem, identifying that the substantially reduced class diversity in target batches (compared to source batches) is the primary contributor to these inaccuracies. Motivated by this insight, we not only explain why TEMA can improve target estimation, but also refine it with an adaptive momentum mechanism to properly incorporate a diverse set of class information.

\section{Preliminary}
\subsection{Problem Setting}
In full test-time adaptation (TTA), we work with a parameterized model, $q_{\mathbf{\theta}}(y|x)$, originally trained on a labeled source dataset $\mathcal{D}_{s} = \{(x,y) \sim p_s( x,y)\} $, where $x \in \mathcal{X}$ is the input and $y$ is the ground truth label from the source class set $\mathcal{Y}$.  We aim to improve the performance of this existing model during inference time for a continually changing target domain $\mathcal{D}_{t} = \{(x,y) \sim p_t(x,y)\} $ in an online fashion, without access to any source data and target labels. Note that for source and target distribution, $p_s(x,y) = p_s(x)p_s(y| x)$ and $p_t(x,y) = p_t(x)p_t(y| x)$,  while  covariate shift exists as $p_s(y| x) = p_t(y|x)$ and $p_s(x) \neq p_t(x)$. In this situation, the pretrained model $q_{\mathbf{\theta}}(y|x)$ cannot closely represent the true, domain-invariant  distribution $p(y|x)$. Therefore, TTA methods concentrate on adjusting $q_{\mathbf{\theta}}(y|x)$ to maximize its predictive performance on the target distribution.

\subsection{Batch Normalization}
Batch normalization (BN) has been widely used in contemporary DNNs. During training, given a mini-batch $\mathcal{B} = \{x_n\}_{n=1}^N$ where $x_n \in \mathbb{R}^{F}$ is a feature vector (with $F$ denoting the number of feature channels and $N$ the batch size), BN normalizes each feature dimension $f$ as follows: 
\begin{equation}
\small
\hat{x}_{n,f} = \frac{x_{n,f}-\mu_{\mathcal{B}, f}}{\sqrt{\sigma_{\mathcal{B}, f}+\epsilon }} \cdot \gamma_f  + \beta_f, 
\label{BN}
\end{equation}
where $\mu_{\mathcal{B},f}$ and $\sigma_{\mathcal{B}, f}$ are the running mean and variance for the $f$-th feature of mini-batch $\mathcal{B}$, respectively. The parameters $\gamma_f$ and $\beta_f$ are the learned scale and shift factors for affine transformation, with $\epsilon$ being a small-offset to avoid division by zero. Meanwhile, the running mean $\mu_\mathcal{B} \in \mathbb{R}^F$ and covariance $\sigma_\mathcal{B} \in \mathbb{R}^F$ are accumulated to estimate the overall mean $\text{E}[X]$ and covariance $\text{Var}[X]$ of the training data:
\begin{small}
\begin{align*}
\text{E}[X] &\Leftarrow  m \cdot \mu_\mathcal{B} + (1-m) \cdot \text{E}[X]\;,\\
\text{Var}[X] &\Leftarrow  m \cdot \sigma_\mathcal{B} + (1-m) \cdot \text{Var}[X] \;.
\end{align*}
\end{small}

\noindent During inference, the conventional method computes BN with the estimated mean $\text{E}[X]$ and covariance $\text{Var}[X]$ from source, while advanced TTA methods that focusing on adapting BN layers adopt a different approach. They utilize statistics computed directly from each test batch to mitigate potential distributional shifts at test-time.

\section{Method}
\subsection{Motivations}
To motivate our approach, we first take a closer look at the target batch statistics during test time, as illustrated in Figure~\ref{fig:banner}. Our observations are as follows:
\textbf{1)} For sufficiently large target batches (e.g., 200), the statistics stabilize and offer an accurate description of target features, thereby aiding in test-time training.
\textbf{2)} As the batch size diminishes, the statistics derived from the test batch become highly volatile and often inaccurate. Many existing techniques falter under these conditions because predictions rely on real-time statistics.
\textbf{3)} Enhancing class diversity enables our approach to notably boost model performance, even when working with mini-batches.
Driven by our findings, we challenge the prevailing notion in TTA research that the accuracy of target estimation primarily depends on batch sample size. Instead, we explore the intricate relationship between normalization statistics and class diversity with the following proposition.

\begin{proposition}\label{prop:quantif}
Given an infinite sample space where each sample is independently and identically distributed (i.i.d) with an equal probability of selection for each category.
Let $M$ denote the number of distinct categories contained within a given batch, and $K$ be the category number in total. For a batch of size $N$, the expected number of unique categories (also referred to as category diversity) is given by:
\begin{equation}\label{eq:quanti}
\small
E(M|N) = \sum_{k=1}^{K} \left[ k \cdot \frac{\mathbf{C}_{N-1}^{k-1}\mathbf{C}_{K}^{k}}{\mathbf{C}_{N+K-1}^{K-1}} \right],
\end{equation}
where $\mathbf{C}$ denotes the combination symbol in Combinatorics.
\end{proposition}

\noindent  Proposition~\ref{prop:quantif} quantifies the relationship between batch size and class diversity. Taking CIFAR-10 dataset as an example with a training batch size $N_s=128$~\cite{hendrycks2019augmix}, we observe $E(M_s|N_s) \approx 9.34$. In case of the test batch size $N_t=200$ --- the default evaluation setting in previous research, $E(M_t|N_t) \approx 9.57$, which is nearly equal to $E(M_s|N_s)$. Such large target batch are advantageous, as their class diversity closely mirrors that of source batches, thereby yielding accurate target statistics.
In real-world scenarios, the test-time batch size can often be much smaller, as in $N_t=2$, a substantial discrepancy arises: $E(M_t|N_t) \approx 1.82$  significantly less than $E(M_s|N_s)$. This discrepancy is not merely a numerical one -- it reflects a fundamental divergence in data distribution that is independent of domain characteristics. This divergence further obscures the inherent domain shifts and results in model degradation. 

\subsection{Accurate Target Estimation with TEMA}
\subsubsection{Enhanced Class Diversity}
We aim to bridge the gap between training and testing in terms of class diversity with a simple yet effective tool, namely, Test-time Exponential Moving Averages (TEMA). 
Specifically, TEMA operates by gradually absorbing statistics from previous batches into current estimations using the following formulas:
\begin{small}
\begin{align}
\mu_{i}^{\text{ema}} &=  m \cdot \mu_{i}^{\text{batch}} + (1 - m) \cdot  \mu_{i-1}^{\text{ema}},\label{eq:tema_formula_1}\\
\sigma_{i}^{\text{ema}} &=  m \cdot \sigma_{i}^{\text{batch}} + (1 - m) \cdot \sigma_{i-1}^{\text{ema}},\label{eq:tema_formula_2}
\end{align}
\end{small}

\noindent where $i$ denotes batch index. $\{\mu^{\text{batch}}_i, \sigma^{\text{batch}}_i\}$ represents the statistics of current batch, and $\{\mu^{\text{ema}}_i, \sigma^{\text{ema}}_i\}$ is the updated parameters for test-time normalization. Here, $m$ refers to the ``momentum", a crucial parameter that controls the amount of information retained from previous mini-batches with values ranging from 0 to 1. Larger $m$ emphasizes the current batch's statistics, while a lower value prioritizes the accumulated historical statistics. 

At its core, TEMA is designed to extend the effective sample pool for the current estimation. This is achieved through a weighted average that integrates information across multiple batches, as specified in the following proposition.
\begin{proposition}\label{prop:momen_range}
Given the iterative rules of TEMA defined in Eq.~3 and~4, it yeilds the $i$-th term as a cumulative sum of the past batch statistics weighted by $w_0 = (1 - m)^i$ for initial batch and $w_t = (1 - m)^{(i-t)} m$ for $t = 1,2,...,i$. Let $\epsilon$ be a threshold defining the effective sample batch, such that only batches with a relative weight $w_t / w_i > \epsilon$ are included. Then, the expanded sample scope for statistical estimation in TEMA can be formally expressed as $\hat{N}_t = \lfloor \log_{1-m} \epsilon \rfloor \cdot N_t$.
\end{proposition}

\begin{proof}
We begin with the observation that, with momentum $m \leq 1$, the weight of each historical batch decays exponentially towards zero as the gap from the current batch $i$ increases. To filter out negligible impact, we introduce a threshold $\epsilon$, such that only batches with a relative weight $w_t / w_i > \epsilon$ are deemed effective. Given the continual influx of batches, 
the criterion simplifies to $(1 - m)^{(i-t)} > \epsilon$, leading to 
$t > i -\log_{1 - m} \epsilon$. 
This condition delineates the impactful batches for TEMA.
As $\log_{1 - m} \epsilon$ is not always an integer, we apply the floor function $\lfloor \cdot \rfloor$ to identify the nearest integer not exceeding this value. Therefore, 
the effective batches can be counted by 
$\lfloor \log_{1 - m} \epsilon \rfloor$, further establishing an expanded sample scope of $\hat{N}_t = \lfloor \log_{1-m} \epsilon \rfloor \cdot N_t$.
\end{proof}

\noindent This enlarged sample pool allows TEMA to seamlessly incorporate a diverse set of class information. Importantly, this is not a heuristic but a deliberate design choice, inspired by our quantitative analysis that reveals the pivotal relationship between batch size and class diversity (see Proposition~\ref{prop:quantif}).

\subsubsection{Adaptive Momentum}
While expanding TEMA's coverage indeed enhances class diversity, it meanwhile risks losing focus on the current batch's data, which could lead to unstable and disrupted normalization. This situation creates a trade-off. Prioritizing the local context of the current batch allows the model to capture fine-grained details, but may suppress information richness. Conversely, emphasizing a broader, global context enriches class diversity at the cost of information timeliness. Recognizing this dilemma, a pioneering contribution of our work is to introduce a versatile strategy that dynamically tailors the momentum in TEMA to navigate this trade-off.

Our primary goal is to identify the optimal momentum that balances two competing objectives: 1) aligning the enhanced class diversity during testing with that during training, and 2) keeping the enlarged sample pool as small as possible. We formalize this balance in the following optimization problem with a trade-off parameter $\lambda > 0$:
\begin{equation}
\small
\mathop{\arg \min }_{m}\ | \frac{E(M_s | N_s)}{E(M_t | \hat{N}_t)} - 1| + \lambda \cdot \frac{\hat{N}_t}{N_s},
\label{objective}
\end{equation}

\noindent where $\hat{N}_t = \lfloor \log_{1-m} \epsilon \rfloor \cdot N_t$ denotes the enlarged sample pool size (as defined in Proposition~\ref{prop:momen_range}), and we respectively fix $\epsilon$ and $\lambda$ at $0.1$ and $0.01$ for the sake of simplicity. To be specific, the first objective term quantifies the discrepancy between the class diversities during training and testing, while the second term severs a regularization that penalizes our enlarged sample pool size.
Given that momentum is a standard hyper-parameter in deep learning, typically selected from the set $\{1, 0.1, 0.01, 0.001\}$, it would be inappropriate in our model to excessively optimize this parameter. We thus adopt a pragmatic approach: employing a grid search over this predefined set to seek the momentum value that minimizes our objective function.

\subsection{Layer-wise Rectification Strategy}

With TEMA's improvement  on target estimation, one may be tempted to completely transition from source to target statistics in batch normalization (BN) layers. However, this may be counterproductive due to the inherent trade-off between target and source information.
Target statistics are pivotal for proper normalization and domain shift alleviation, but they can also render the model unstable due to non-stationary parameters. On the other hand, source statistics, while not ideally suited for target domains, play a crucial role in preserving the model's stability and robustness.

To navigate this complex landscape, we propose a novel layer-wise rectification strategy that harmonizes the target statistics with the source. Central to this approach is the use of inter-domain divergence in each BN layer as a metric to balance their contributions to the final normalization statistics. Our strategy is designed to be adaptive: it increases reliance on the source when the disparity between source and target distributions is significant, thereby promoting model stability. Otherwise, when the domain differences are more nuanced, our approach prioritizes the target statistics, ensuring precise adaptation to subtle domain shifts. 

The overall pipeline of the proposed strategy is detailed in Algorithm~\ref{alg},
where we perform test-time batch normalization using combined statistics according to:
\begin{small}
\begin{align}
\mu^{(l)} &=  \alpha^{(l)} \cdot \mu_{s}^{(l)} + (1 - \alpha^{(l)}) \cdot  \mu_{t}^{(l)} \label{interpolation_1}\;, \\
(\sigma^{(l)})^{2} &=  \alpha^{(l)} \cdot (\sigma_{s}^{(l)})^{2} + (1 - \alpha^{(l)}) \cdot  (\sigma_{t}^{(l)})^{2}\\ \notag &+ \alpha^{(l)} \cdot (1-\alpha^{(l)})(\mu_s^{(l)}-\mu_t^{(l)})^{2} \label{interpolation_2}\;.
\end{align}
\end{small}

\noindent In this context, $\alpha^{(l)}$ serves as a trade-off coefficient and is tied with inter-domain distribution divergence. $\{\mu_{s}^{(l)},\sigma_{s}^{(l)}\}$ denote the source statistics for $l$-th BN layer, $\{\mu_{t}^{(l)},\sigma_{t}^{(l)}\}$ refers to the target statistics estimated by TEMA.

One should note that in previous research, $\alpha$ is a manually tuned hyper-parameter for different tasks and is typically fixed for different incoming samples. While our strategy stands out from traditional ones by eliminating hyper-parameter tuning, 
enhancing the applicability of our approach in diverse real-world scenarios.

\section{Experiments and Analysis}

\begin{table*}[t]
\renewcommand{\arraystretch}{1.0} 
\setlength{\tabcolsep}{4.3pt} 
\begin{tabular}{cc|ccccccc|ccccccc}
\hline
\multicolumn{2}{c|}{}                                                                                                    & \multicolumn{7}{c|}{CIFAR-10-C }                                                                                                                                                                                                                                                                                   & \multicolumn{7}{c}{CIFAR-100-C}                                                                                                                                                                                                                                                                                   \\ \cline{3-16} 
\multicolumn{2}{c|}{\multirow{-2}{*}{Continual}}                                                                            & 200                                    & 64                                     & 16                                     & 4                                      & 2                                      & \multicolumn{1}{c|}{1}                                      & Avg.                                   & 200                                    & 64                                     & 16                                     & 4                                      & 2                                      & \multicolumn{1}{c|}{1}                                      & Avg.                                   \\ \hline \hline
\multicolumn{1}{c|}{}                                                                                      & Source      & 43.50                                  & 43.50                                  & 43.50                                  & 43.50                                  & 43.50                                  & \multicolumn{1}{c|}{43.50}                                  & 43.50                                  & 46.45                                  & 46.45                                  & 46.45                                  & 46.45                                  & 46.45                                  & \multicolumn{1}{c|}{46.45}                                  & 46.45                                  \\ \cline{1-16}

\multicolumn{1}{c|}{}                                                                                      & TENT        & 19.55                                  & 26.32                                  & 74.07                                  & 85.07                                  & 88.69                                  & \multicolumn{1}{c|}{90.00}                                  & 63.95                                  & 61.12                                  & 86.49                                  & 96.07                                  & 98.40                                  & 98.79                                  & \multicolumn{1}{c|}{99.00}                                  & 89.98                                  \\
\multicolumn{1}{c|}{}                                                                                      & CoTTA       & 16.24                                  & 17.65                                  & 34.36                                  & 78.88                                  & 87.79                                  & \multicolumn{1}{c|}{90.00}                                  & 54.15                                  & 32.68                                  & 34.30                                  & 47.52                                  & 92.62                                  & 97.84                                  & \multicolumn{1}{c|}{98.96}                                  & 67.32                                  \\
\multicolumn{1}{c|}{}                                                                                      & SAR         & 20.40                                  & 20.74                                  & 22.89                                  & 31.35                                  & 40.32                                  & \multicolumn{1}{c|}{89.83}                                  & 37.59                                  & 31.90                                  & 35.89                                  & 54.84                                  & 66.08                                  & 73.24                                  & \multicolumn{1}{c|}{98.91}                                  & 60.14                                  \\
\multicolumn{1}{c|}{}                                                                                      & AdaCont. & 18.50                                  & 17.41                                  & 19.69                                  & 35.01                                  & 63.81                                  & \multicolumn{1}{c|}{31.55}                                  & 31.00                                  & 33.61                                  & 35.40                                  & 55.04                                  & 89.45                                  & 96.16                                  & \multicolumn{1}{c|}{62.67}                                  & 62.06                                  \\
\multicolumn{1}{c|}{\multirow{-5}{*}{\begin{tabular}[c]{@{}c@{}}\rotatebox{90}{TR}\end{tabular}}} & ETA         & 17.64                                  & 20.01                                  & 30.60                                   & 56.78                                  & 83.24                                  & \multicolumn{1}{c|}{89.83}                                  & 49.68                                  & 32.31                                  & 35.17                                  & 44.72                                  & 88.22                                  & 98.96                                  & \multicolumn{1}{c|}{98.91}                                  & 66.38                                  \\ \hline\hline

\multicolumn{1}{c|}{}                                                                                      & TBN        & \underline{20.35}                                  & 20.82                                  & 23.06                                  & 31.62                                  & 38.57                                  & \multicolumn{1}{c|}{89.83}                                  & 37.38                                  & 35.50                                  & 36.29                                  & 39.67                                  & 52.73                                  & 73.24                                  & \multicolumn{1}{c|}{98.91}                                  & 56.06                                  \\
\multicolumn{1}{c|}{}                                                                                      & $\alpha$-BN            & 30.60                                  & 30.67                                  & 30.89                                  & 31.89                                  & 32.91                                  & \multicolumn{1}{c|}{\underline{34.47}}                                  & 31.91                                  & 37.02                                  & 38.27                                  & \underline{35.92}                                  & 37.17                                  & \underline{37.25}                                  & \multicolumn{1}{c|}{\underline{41.18}}                                  & 37.80                                  \\
\multicolumn{1}{c|}{}                                                                                      & AdaptBN  & 20.36                                  & \underline{20.71}                                  & \underline{21.98}                                  & \underline{26.79}                                  & \underline{32.19}                                  & \multicolumn{1}{c|}{37.52}                                  & \underline{26.59}                                  & \underline{35.40}                                  & \textbf{35.78}                                  & \underline{35.92}                                  & \underline{37.14}                                  & 39.23                                  & \multicolumn{1}{c|}{41.85}                                  & \underline{37.55}                                  \\
\multicolumn{1}{c|}{}                                                                                      & LAME        & 64.52                                  & 57.66                                  & 47.70                                  & 44.38                                  & 43.75                                  & \multicolumn{1}{c|}{90.00}                                  & 58.00                                  & 98.49                                  & 73.83                                  & 47.59                                  & 46.64                                  & 46.50                                  & \multicolumn{1}{c|}{99.00}                                  & 68.68                                  \\  
\multicolumn{1}{c|}{\multirow{-5}{*}{\begin{tabular}[c]{@{}c@{}}\rotatebox{90}{TF}\end{tabular}}}      & \cellcolor[HTML]{EFEFEF}\textbf{Ours}        & \cellcolor[HTML]{EFEFEF}\textbf{20.20} & \cellcolor[HTML]{EFEFEF}\textbf{20.57} & \cellcolor[HTML]{EFEFEF}\textbf{20.74} & \cellcolor[HTML]{EFEFEF}\textbf{21.45} & \cellcolor[HTML]{EFEFEF}\textbf{20.91} & \multicolumn{1}{c|}{\cellcolor[HTML]{EFEFEF}\textbf{21.05}} & \cellcolor[HTML]{EFEFEF}\textbf{20.82} & \cellcolor[HTML]{EFEFEF}\textbf{34.63} & \cellcolor[HTML]{EFEFEF}{\underline{36.11}} & \cellcolor[HTML]{EFEFEF}\textbf{35.31} & \cellcolor[HTML]{EFEFEF}\textbf{36.02} & \cellcolor[HTML]{EFEFEF}\textbf{36.32} & \multicolumn{1}{c|}{\cellcolor[HTML]{EFEFEF}\textbf{39.30}} & \cellcolor[HTML]{EFEFEF}\textbf{36.28} \\ \hline

\end{tabular}
\caption{Continual adaptation on corruption benchmark CIFAR-10-C/CIFAR-100-C. Error rate ($\downarrow$) averaged over 15 corruptions with severity level 5 for each test batch size (200/64/16/4/2/1).}
\label{continual_cifar}
\end{table*}

\begin{table}[t]
\renewcommand{\arraystretch}{1.0} 
\setlength{\tabcolsep}{3.5pt} 
\begin{tabular}{cc|rrllll}
\hline
\multicolumn{2}{c|}{}                                                                                       & \multicolumn{6}{c}{ImageNet-C}                                                                                                                                                                                                                                                                                                                                                                                                                               \\ \cline{3-8} 
\multicolumn{2}{c|}{\multirow{-2}{*}{Continual}}                                                            & \multicolumn{1}{c}{64}                                                   & \multicolumn{1}{c}{16}                                                   & \multicolumn{1}{c}{4}                                                             & \multicolumn{1}{c}{2}                                         & \multicolumn{1}{c|}{1}                                                             & \multicolumn{1}{c}{Avg.}                                      \\ \hline\hline
\multicolumn{1}{c|}{}                        & Source                                                       & \multicolumn{1}{c}{82.00}                                                & \multicolumn{1}{c}{82.00}                                                & \multicolumn{1}{c}{82.00}                                                         & \multicolumn{1}{c}{82.00}                                     & \multicolumn{1}{c|}{82.00}                                                         & \multicolumn{1}{c}{82.00}                                     \\ \cline{1-8}
\multicolumn{1}{c|}{}                        & TENT                                                         & 62.60                                                                    & 91.99                                                                    & 99.74                                                                             & 99.85                                                         & \multicolumn{1}{l|}{99.90}                                                         & 90.82                                                         \\
\multicolumn{1}{c|}{}                        & CoTTA                                                        & 63.10                                                                    & 84.35                                                                    & 99.79                                                                             & 99.88                                                          & \multicolumn{1}{l|}{99.79}                                                         & 89.38                                                         \\
\multicolumn{1}{c|}{}                        & SAR                                                          & 68.04                                                                    & 62.62                                                                    & 79.19                                                                             & 92.71                                                         & \multicolumn{1}{l|}{99.19}                                                           & 80.35                                                         \\
\multicolumn{1}{c|}{}                        & AdaCont.                                                  & 66.83                                                                    & 92.02                                                                    & 98.40                                                                              & 99.66                                                         & \multicolumn{1}{l|}{99.88}                                                         & 91.36                                                         \\
\multicolumn{1}{c|}{\multirow{-5}{*}{\begin{tabular}[c]{@{}c@{}}\rotatebox{90}{TR}\end{tabular}}}  & ETA          & \multicolumn{1}{l}{58.68}                        & \multicolumn{1}{l}{69.65}                        & 98.81                                                     & 93.14                                 & \multicolumn{1}{l|}{99.19}                                 & 83.89                                 \\ \hline\hline
\multicolumn{1}{c|}{}                        & TBN                                                         & 68.60                                                                    & 70.79                                                                    & 83.10                                                                             & 92.74                                                            & \multicolumn{1}{l|}{99.22}                                                         & 82.89                                                         \\
\multicolumn{1}{c|}{}                        & $\alpha$-BN                                                     & \textbf{63.38}                                                           & \textbf{63.75}                                                           & \underline{64.98}                                                                             & \underline{66.61}                                                         & \multicolumn{1}{l|}{68.90}                                                         & \textbf{65.53}                                                         \\
\multicolumn{1}{c|}{}                        & AdaptBN                                                   & \underline{64.13}                                                                    & 65.41                                                                    & 65.63                                                                             & 66.72                                                         & \multicolumn{1}{l|}{\textbf{68.12}}                                                         & 66.00                                                         \\
\multicolumn{1}{c|}{}                        & LAME                                                         & 93.50                                                                    & 73.59                                                                    & 73.22                                                                             & 73.18                                                         & \multicolumn{1}{l|}{99.90}                                                         & 82.68                                                         \\
\multicolumn{1}{c|}{\multirow{-5}{*}{\begin{tabular}[c]{@{}c@{}}\rotatebox{90}{TF}\end{tabular}}}  & \cellcolor[HTML]{EFEFEF}{\color[HTML]{000000} \textbf{Ours}} & \multicolumn{1}{l}{\cellcolor[HTML]{EFEFEF}{\color[HTML]{000000} 64.15}} & \multicolumn{1}{l}{\cellcolor[HTML]{EFEFEF}{\color[HTML]{000000} \underline{64.90}}} & \multicolumn{1}{r}{\cellcolor[HTML]{EFEFEF}{\color[HTML]{000000} \textbf{64.87}}} & \cellcolor[HTML]{EFEFEF}{\color[HTML]{000000} \textbf{66.01}} & \multicolumn{1}{l|}{\cellcolor[HTML]{EFEFEF}{\color[HTML]{000000} \underline{68.79}}} & \cellcolor[HTML]{EFEFEF}{\color[HTML]{000000} \underline{65.74}} \\ \hline
\end{tabular}
\caption{Continual adaptation on  ImageNet-C. }
\label{imagenet}
\end{table}

\subsection{Datasets and Model Architectures}
We evaluate our approach on CIFAR-10-C, CIFAR-100-C, and ImageNet-C~\cite{imagenetc}, which were initially designed to benchmark robustness of classification networks.
All the corruption datasets are obtained by applying 15 kinds of corruption with 5 different degrees of severity on their clean test images of original datasets. CIFAR-10/CIFAR-100 originally has 10,000 test images, and ImageNet has 50,000 test images, which fall into 10/100/1000 categories, respectively. This results in a total of 150,000 test data for CIFAR-10-C/CIFAR-100-C, and 750,000 test data for ImageNet-C for each severity level. 

Following the previous methods~\cite{wang2021tent,wang2022continual},
we obtain the pretrained model from RobustBench benchmark~\cite{croce2021robustbench}, including the WildResNet-28~\cite{wideresnet} for  CIFAR-10-C, ResNeXt-29~\cite{resnext}  for CIFAR-100-C (both pretrained by \citet{hendrycks2019augmix}), and ResNet-50~\cite{he2016deep} for ImageNet-C (standard pretrained). All experiments are conducted on an RTX-3090 GPU.

\subsection{Evaluation Settings}
To show that our method performs robust on various test batch sizes, we conduct
experiments with test batch sizes of 200, 64, 16, 4, 2, and 1 for CIFAR-10/100-C, and   64, 16, 4, 2, and 1 for ImageNet-C.
Note that 200 for CIFAR-10/100-C and 64 for ImageNet-C are the widely-used evaluation batch size.
Following~\citet{dobler2023robust},
we evaluate our methods in the following settings:

\noindent \textbf{Continual}:
The model adapts to a sequence of  test domains in an online manner without knowing when it changes. 

\noindent \textbf{Mixed domains}: The model adapts to one long test sequence where consecutive test samples are likely to originate from different domains. 

\noindent \textbf{Gradual}: The model adapts to  a sequence of gradually increasing/decreasing domain shifts. This is formulated as:
\begin{scriptsize}
\noindent $\underbrace{{...2 \,\rightarrow\! 1}}_{\text{type}\, t-1} \xrightarrow[\text{change}]{\text{type}} \underbrace{{1\! \rightarrow\! 2\! \rightarrow\! 3\! \rightarrow\! 4\! \rightarrow\! 5\! \rightarrow\! 4\! \rightarrow\! 3\! \rightarrow\! 2\! \rightarrow\! 1}}_{\text{corruption}\, t,  \text{changing}\, \text{severity}} \xrightarrow[\text{change}]{\text{type}}\underbrace{{1\! \rightarrow\! 2 ...}}_{\text{type}\, t+1}$
\end{scriptsize}

\begin{table*}[t]
\renewcommand{\arraystretch}{1.0} 
\setlength{\tabcolsep}{4.3pt} 
\begin{tabular}{cc|lllllll|lllllll}
\hline
\multicolumn{2}{c|}{}                                                                                                                                                     & \multicolumn{7}{c|}{Mixed Domain}                                                                                                                                                                                                                                                                                                                                                                                                                                                 & \multicolumn{7}{c}{Gradual}                                                                                                                                                                                  \\ \cline{3-16}
\multicolumn{2}{c|}{\multirow{-2}{*}{CIFAR-10C}}                                                                                                                          & \multicolumn{1}{c}{200}                                       & \multicolumn{1}{c}{64}                                        & \multicolumn{1}{c}{16}                                        & \multicolumn{1}{c}{4}                                         & \multicolumn{1}{c}{2}                                         & \multicolumn{1}{c|}{1}                                                             & \multicolumn{1}{c|}{Avg.}                                     & \multicolumn{1}{c}{200}                                      & \multicolumn{1}{c}{64}                                        & \multicolumn{1}{c}{16}                                        & \multicolumn{1}{c}{4}                                       & \multicolumn{1}{c}{2}                                       & \multicolumn{1}{c|}{1}                                                           & \multicolumn{1}{c}{Avg.}                                                         \\ \hline\hline
\multicolumn{1}{c|}{}                                                                                      & Source                                                       & \multicolumn{1}{c}{43.50}                                     & \multicolumn{1}{c}{43.50}                                     & \multicolumn{1}{c}{43.50}                                     & \multicolumn{1}{c}{43.50}                                     & \multicolumn{1}{c}{43.50}                                     & \multicolumn{1}{c|}{43.50}                                                         & \multicolumn{1}{c|}{43.50}                                    & \multicolumn{1}{r}{24.66}                                    & \multicolumn{1}{r}{24.66}                                     & \multicolumn{1}{r}{24.66}                                     & \multicolumn{1}{r}{24.66}                                   & \multicolumn{1}{r}{24.66}                                   & \multicolumn{1}{r|}{24.66}                                                       & 24.66                            \\ \cline{1-16}
\multicolumn{1}{c|}{}                                                                                      & TENT                                                         & 39.75                                                         & 57.14                                                         & 80.30                                                         & 88.39                                                         & 89.26                                                         & \multicolumn{1}{l|}{90.00}                                                         & 74.14                                                         & 26.32                                                        & 56.19                                                         & 81.79                                                         & 90.84                                                       & 87.96                                                       & \multicolumn{1}{l|}{90.00}                                                       & 72.18                                                                     \\
\multicolumn{1}{c|}{}                                                                                      & CoTTA                                                        & 32.37                                                         & 31.22                                                         & 51.30                                                         & 85.04                                                         & 86.97                                                         & \multicolumn{1}{l|}{89.85}                                                         & 62.79                                                         & 11.16                                                        & 16.83                                                         & 59.38                                                         & 86.96                                                         & 89.68                                                       & \multicolumn{1}{l|}{90.00}                                                       & 59.00  \\
\multicolumn{1}{c|}{}                                                                                      & SAR                                                          & 33.74                                                         & 33.97                                                         & 35.38                                                         & 41.13                                                         & 48.86                                                         & \multicolumn{1}{l|}{89.83}                                                         & 47.15                                                         & 13.97                                                        & 14.43                                                         & 16.51                                                         & 25.28                                                       & 32.47                                                       & \multicolumn{1}{l|}{89.75}                                                       & 32.07\\
\multicolumn{1}{c|}{}                                                                                      & AdaCont.                                                  & 26.12                                                         & 23.85                                                         & 23.58                                                         & 35.57                                                         & 62.50                                                         & \multicolumn{1}{l|}{46.01}                                                         & 36.27                                                         & 12.33                                                        & 13.66                                                         & 19.93                                                         & 33.13                                                       & 54.17                                                       & \multicolumn{1}{l|}{22.33}                                                         & 25.93  \\
\multicolumn{1}{c|}{\multirow{-5}{*}{\begin{tabular}[c]{@{}c@{}}\rotatebox{90}{TR}\end{tabular}}} & ETA 
                       & 27.92                                 & 34.81                                 & 55.07                                 & 69.62                                 & 85.43                                 & \multicolumn{1}{l|}{89.83}                                 & 60.45                                 & 16.63                                & 21.82                                 & 57.51                                 & 77.46                               & 84.59                               & \multicolumn{1}{l|}{89.75}                               & 57.96                            \\ \hline\hline
\multicolumn{1}{c|}{}                                                                                      & TBN                                                         & \textbf{33.77}                                                         & \textbf{34.05}                                                         & \underline{35.51}                                                         & 40.72                                                         & 45.05                                                         & \multicolumn{1}{l|}{89.83}                                                         & 46.49                                                         & 13.64                                                        & 14.10                                                         & 16.21                                                         & 25.10                                                       & 31.66                                                       & \multicolumn{1}{l|}{89.75}                                                       & 31.74 \\
\multicolumn{1}{c|}{}                                                                                      & $\alpha$-BN                                                            & 39.85                                                         & 39.85                                                         & 39.59                                                         & 38.80                                                         & 37.46                                                         & \multicolumn{1}{l|}{\textbf{34.47}}                                                         & 38.34                                                         & 18.05                                                        & 18.12                                                         & 18.25                                                         & 18.87                                                       & 19.49                                                       & \multicolumn{1}{l|}{\underline{20.13}}                                                       & 18.82 \\
\multicolumn{1}{c|}{}                                                                                      & AdaptBN                                                   & \underline{33.86}                                                         & \underline{34.35}                                                         & 35.69                                                         & \underline{36.96}                                                         & \underline{37.10}                                                         & \multicolumn{1}{l|}{37.52}                                                         & \underline{35.91}                                                         & \underline{13.63}                                                        & \underline{13.96}                                                         & \underline{14.86}                                                         & \underline{17.20}                                                       & \underline{19.23}                                                       & \multicolumn{1}{l|}{21.57}                                                       & \underline{16.74} \\
\multicolumn{1}{c|}{}                                                                                      & LAME                                                         & 75.09                                                         & 52.30                                                         & 44.87                                                         & 43.74                                                         & 43.55                                                         & \multicolumn{1}{l|}{90.00}                                                         & 58.26                                                         & 34.22                                                        & 31.21                                                         & 26.69                                                         & 25.13                                                       & 24.84                                                       & \multicolumn{1}{l|}{90.00}                                                       & 38.68                    \\
\multicolumn{1}{c|}{\multirow{-5}{*}{\begin{tabular}[c]{@{}c@{}}\rotatebox{90}{TF}\end{tabular}}}     & \cellcolor[HTML]{EFEFEF}{\color[HTML]{000000} \textbf{Ours}} & \cellcolor[HTML]{EFEFEF}{\color[HTML]{000000} {34.18}} & \cellcolor[HTML]{EFEFEF}{\color[HTML]{000000} {\underline{34.35}}} & \cellcolor[HTML]{EFEFEF}{\color[HTML]{000000} \textbf{34.48}} & \cellcolor[HTML]{EFEFEF}{\color[HTML]{000000} \textbf{34.79}} & \cellcolor[HTML]{EFEFEF}{\color[HTML]{000000} \textbf{35.59}} & \multicolumn{1}{l|}{\cellcolor[HTML]{EFEFEF}{\color[HTML]{000000} {\underline{37.33}}}} & \cellcolor[HTML]{EFEFEF}{\color[HTML]{000000} \textbf{35.12}} & \cellcolor[HTML]{EFEFEF}{\color[HTML]{000000} \textbf{13.50}} & \cellcolor[HTML]{EFEFEF}{\color[HTML]{000000} \textbf{13.83}} & \cellcolor[HTML]{EFEFEF}{\color[HTML]{000000} \textbf{14.03}} & \cellcolor[HTML]{EFEFEF}{\color[HTML]{000000} \textbf{14.38}} & \cellcolor[HTML]{EFEFEF}{\color[HTML]{000000} \textbf{13.71}} & \multicolumn{1}{l|}{\cellcolor[HTML]{EFEFEF}{\color[HTML]{000000} \textbf{13.99}}} & \cellcolor[HTML]{EFEFEF}{\color[HTML]{000000} \textbf{13.91}}                  \\ \cline{1-16}
\end{tabular}
\caption{Mixed Domain/ Gradual adaptation on corruption benchmark CIFAR-10-C.  For gradual setting, error rate (↓) averaged over 15 corruptions and severity level 1-5.}
\label{mixed}
\end{table*}

\begin{algorithm}[t]
\caption{Layer-wise Rectification Strategy}
\begin{algorithmic}[1]
\REQUIRE Test step $T:=0$; test stream sample $D_{test}$; Source pretrained model $q_{\mathbf{\theta}}(y|x)$ with source statistics $\{\mu_s^{(l)},\sigma_s^{(l)}\}$ for each $l$-th BN layer; Global prior $\mathcal{A}^{'}_{T} = [\alpha^{(1),\text{ema}}_T,\alpha^{(2),\text{ema}}_T,..., \alpha^{(L),\text{ema}}_T]$ initialized with  $\alpha^{(l),\text{ema}}_0 = 0$, $L$ denotes the BN layer number. 
\WHILE{the test batch arrives}
\STATE\textcolor{gray}{// Stage 1: Model the source and target distribution and calculate divergence.} 
    \FOR{each BN layer $l$}
    
        \STATE Calculate the batch statistics $\mu_t^{(l)},\sigma_t^{(l)} \in \mathbb{R}^F$.
        \STATE Combine the normalization statistics with  $\alpha^{(l)} = \mathcal{A}^{'}_{T}[l]$ as Eq.~(\ref{interpolation_1}) and Eq.~(7).
        \STATE Model the distributions as MultiNormal Gaussians: $p_s^{(l)} = \mathcal{N}(\mu_{s}^{(l)},\sigma_{s}^{(l)})$, $p_t^{(l)} = \mathcal{N}(\mu_{t}^{(l)},\sigma_{t}^{(l)})$.
        \STATE Calculate the distribution divergence as: \\
        $D_{\text{KL}}^{(l)}(p_s^{(l)}, p_t^{(l)}) = 1/2 \cdot  p_s^{(l)}  \log\left({p_s^{(l)}}/{p_t^{(l)}}\right) + 1/2 \cdot p_t^{(l)} \log\left({p_t^{(l)}}/{p_s^{(l)} }\right)\;.$ 
    \ENDFOR
    
\STATE\textcolor{gray}{// Stage 2: Obtain the relative divergence $\mathcal{A}$.}
    \STATE Normalization the divergence over $L$ layers, Clip the value to [-1,1], and scale to [0,1]:\\
 $\mathcal{D} = [D^{(1)}_{\text{KL}},D^{(2)}_{\text{KL}},...,D^{(L)}_{\text{KL}}]\;.$ \\
 $\mathcal{A} = \gamma \cdot ({\texttt{Clip}(\texttt{Norm}(\mathcal{D})+1})/{2}\;.$ 
 \textcolor{gray}{// $\gamma$ is set as $1/2$ for emphasizing the target part.}
    
\STATE\textcolor{gray}{// Stage 3: Prediction.}
    \STATE Combine the normalization statistics with  $\alpha =  \mathcal{A}$ and obtain final predictions.
    
\STATE\textcolor{gray}{// Stage 4: Update the global prior $\mathcal{A}^{'}_{T+1}$.}
    \STATE $\mathcal{A}^{'}_{T+1} = \tau\cdot\mathcal{A} + (1-\tau)\cdot\mathcal{A}^{'}_{T}\;. $\textcolor{gray}{// $\tau$ is set as $0.1$.}
  \STATE  $T += 1$.
\ENDWHILE
\end{algorithmic}
\label{alg}
\end{algorithm}

\subsection{Compared Methods}
We compare the following state-of-the-art training-free (TF) methods. TBN~\cite{nado2020evaluating} re-estimates the batch normalization statistics from the test
data. $\alpha$-BN~\cite{you2021test}  combines the source and the test batch statistics with a pre-defined hyperparameter. AdaptiveBN~\cite{schneider2020adaptBN} proposes to reduce the covariate shift  with the interpretation for normalization statistics of a $N$ pseudo sample size  for samples from the training set. Different from them, LAME~\cite{lame} shifted the focus from the model’s parameters to the  output probabilities via  laplacian adjusted maximum-likelihood estimation.

Moreover, we also include recent training-required (TR) methods for reference.
TENT~\cite{wang2021tent} focuses on optimizing batch normalization by minimizing the entropy of the model's predictions. The study of CoTTA~\cite{wang2022continual} delves into long-term test-time adaptation in environments that continually change. Methods like ETA~\cite{niu2022efficient} and SAR~\cite{niu2022towards} are geared towards excluding samples that are deemed unreliable and redundant from the optimization process. AdaContrast~\cite{chen2022contrastive} leverages contrastive learning to enhance feature learning, incorporating a mechanism for refining pseudo-labels. 

\subsection{Results}
Table~\ref{continual_cifar},\ref{imagenet} show the error rates on three corruption benchmark datasets under a continual evaluation setting.
In Table~\ref{mixed}, we report the error rates on CIFAR-10-C under mixed-domain adaptation and gradual changing shifts, respectively.

\noindent\textbf{Our method consistently outperforms other approaches.}
As presented in Table~\ref{continual_cifar},\ref{imagenet}, our method shows a robust and superior performance on all corruption benchmark datasets. The results highlight the consistency of our method in delivering the lowest error rates. Notably, on the CIFAR-10-C dataset, our approach achieved an average error rate of 20.82\% across all batch sizes. This is remarkably lower compared to the next-best performing method, AdaptBN, with an average error rate of 26.59\%.

\noindent\textbf{Our method maintains a high level of performance even with the reduction of batch size.}
The robustness of our method is further illustrated by the minimal increase in error rates as the batch size decreases. For example, on the CIFAR-10-C dataset, the error rate increases marginally from 20.20\% with a batch size of 200 to 21.05\% with a batch size of 1 (Table~\ref{continual_cifar}). 
While for TENT, CoTTA and other training-required methods, this degradation
is significant due to the dependence on the inaccurate test batch statistics. Note that compared with $\alpha$-BN and AdaptBN which leverage large source statistics (80\% or more under mini-batches),  their performance are limited with mini-batches though they do avoid degradation.

\noindent\textbf{Our method demonstrates remarkable stability across different corruption levels and various practical scenarios.}
The results from Table~\ref{mixed} underscore our method's robustness in diverse (mixed domain) and changing (gradual) environments. Our method still stands out with its strong performance, especially under small batch sizes – conditions that can challenge many methods. While other methods tend to display considerable fluctuations in error rates across different corruption levels (especially for the optimization-based methods), our approach manages to maintain relatively steady performance.
The solid empirical results validate the proposed approach's robustness and adaptability, highlighting its potential value for real-world deployment where these types of changes are commonplace.


\begin{figure*}[t]
\centering
  \begin{subfigure}{0.3\textwidth}
  
    \includegraphics[width=\linewidth]{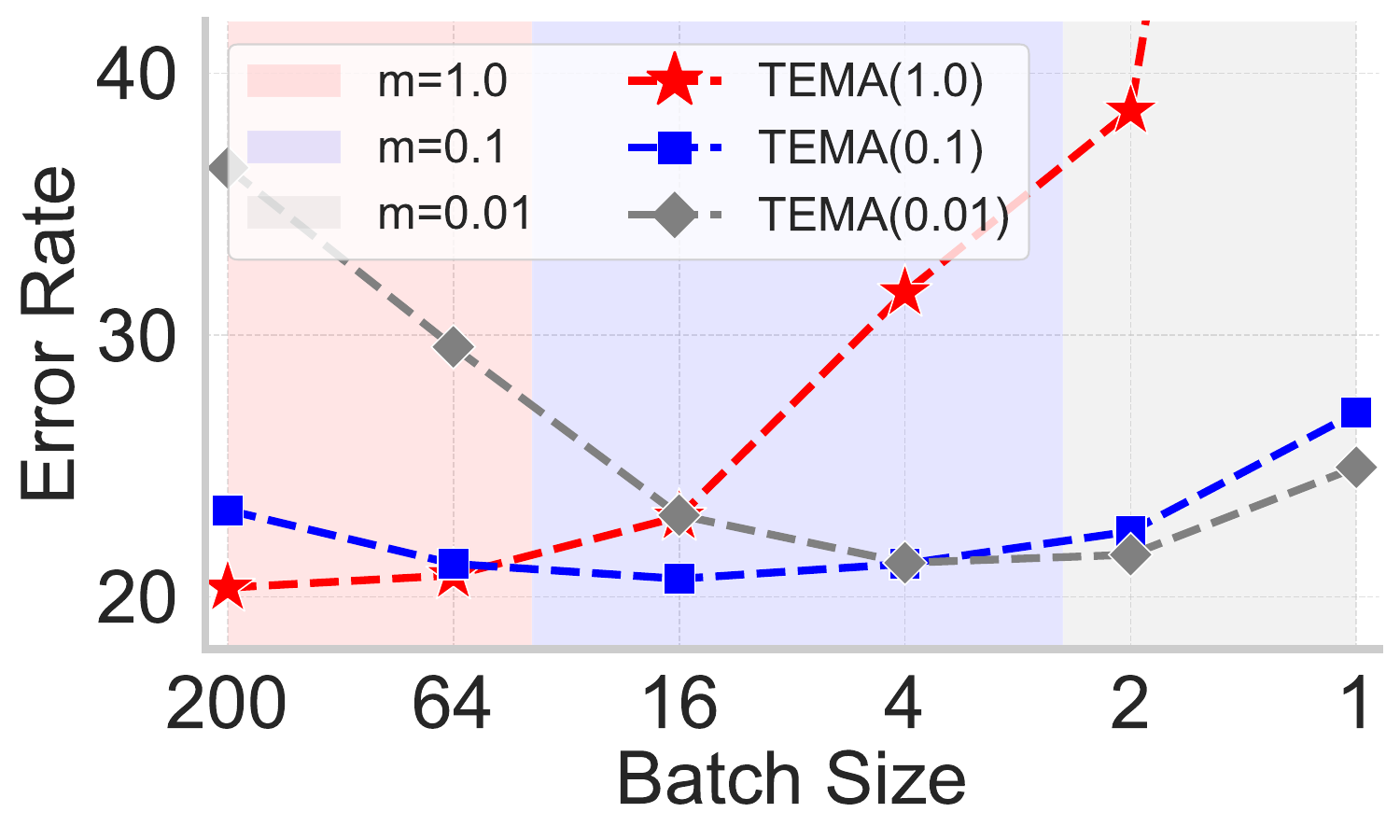}
    \caption{CIFAR-10-C}
    
  \end{subfigure}
  \begin{subfigure}{0.3\textwidth}
  
    \includegraphics[width=\linewidth]{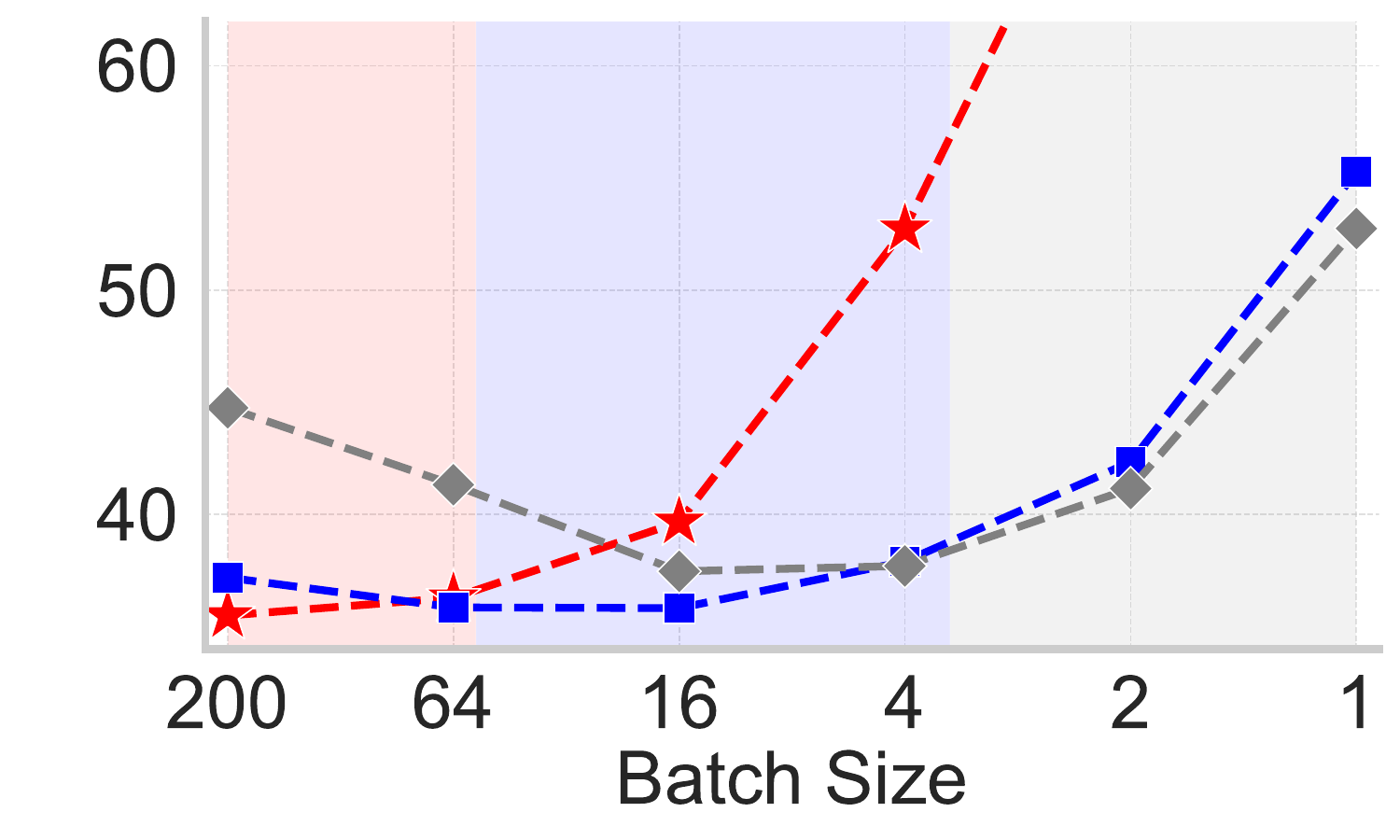}
    \caption{CIFAR-100-C}
    
  \end{subfigure}
  \begin{subfigure}{0.3\textwidth}
      
    \includegraphics[width=\linewidth]{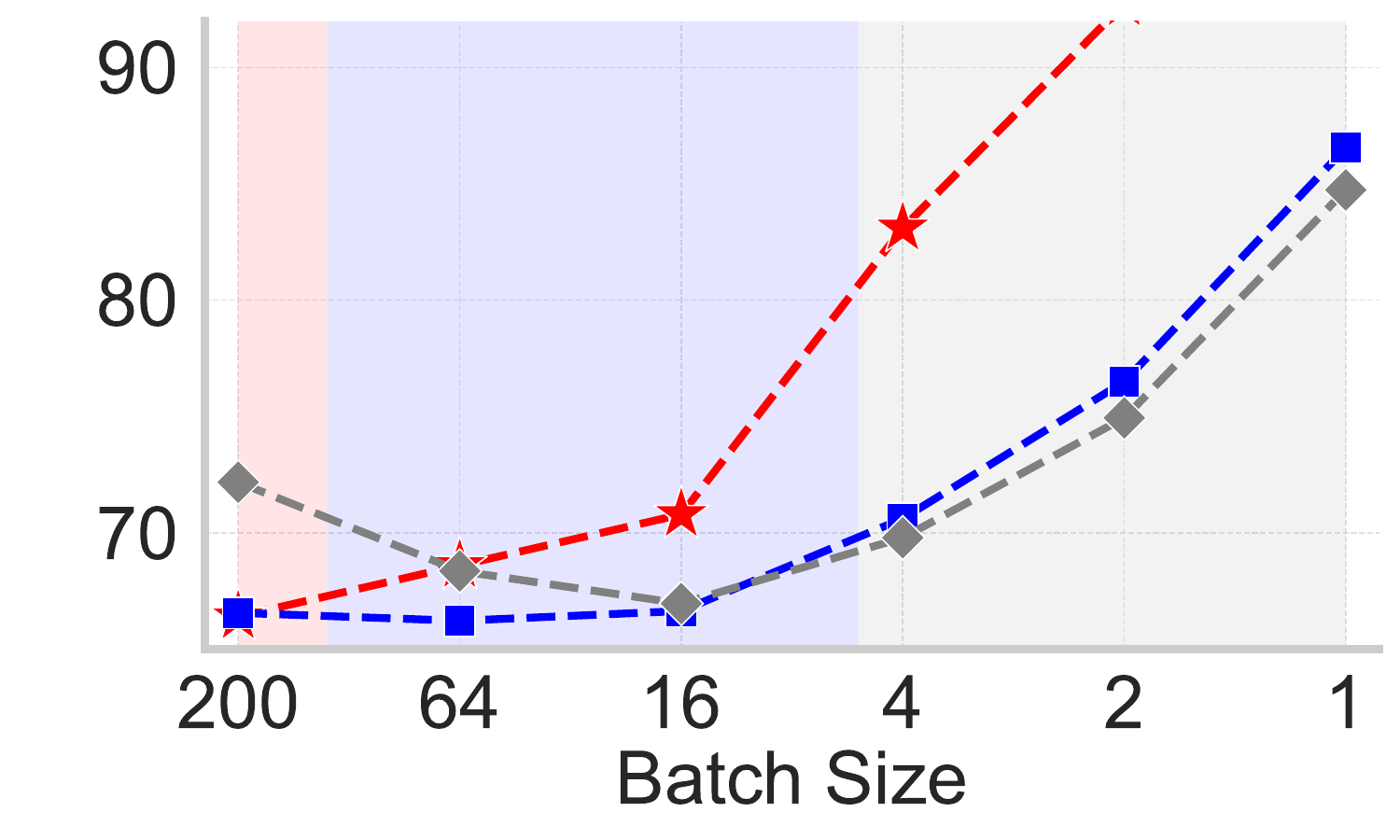}
    \caption{IMAGENET-C}

  \end{subfigure}
  \caption{Momentum analysis for TEMA on three benchmarks under continual setting with different test batch size. Red, blue and  grey regions represent the calculated part where momentum should be set to $m=1.0,0.1,0.01$ according to Eq.~(\ref{objective}). Lines plot the experimental performance of TEMA(m=1.0)/TBN, TEMA(m=0.1), and TEMA(m=0.01).}
\label{fig:momentum}
\end{figure*}

\subsection{Momentum Analysis for TEMA}
We conduct experiments using benchmark datasets to evaluate the proposed adaptive momentum strategy. The process is visualized in Figure~\ref{fig:momentum}. 
Note that the intersections of the lines do not imply the two methods exhibit identical performance under this configuration. Instead, they represent trends for different methods, serving to illustrate our estimations correspond to the actual performance observed under the tested batch size.
The results not only showcase the congruence between our estimated values and empirical performance but also highlight our achievement in effectively balancing the class diversity and timeliness of target statistics.

\begin{table}[t]
\renewcommand{\arraystretch}{1.0} 
\setlength{\tabcolsep}{3pt} 
\begin{small}
\begin{tabular}{c|l|rrrllrl}
\hline
                              & \multicolumn{1}{c|}{}                         & \multicolumn{7}{c}{Test batch size}                                                                                                                                                                                                                                                                                                                                                                                                                                      \\ \cline{3-9} 
\multirow{-2}{*}{}     & \multicolumn{1}{c|}{\multirow{-2}{*}{Method}} & \multicolumn{1}{c}{200}                           & \multicolumn{1}{c}{64}                                     & \multicolumn{1}{c}{16}                                     & \multicolumn{1}{c}{4}                                                    & \multicolumn{1}{c}{2}                                                             & \multicolumn{1}{c|}{1}                                                             & \multicolumn{1}{c}{Avg.}               \\ \hline\rule{0pt}{8pt}
       & \cellcolor[HTML]{FFFFFF}Baseline                                         & \cellcolor[HTML]{FFFFFF}43.50                     & \cellcolor[HTML]{FFFFFF}43.50                              & \cellcolor[HTML]{FFFFFF}43.50                              & \multicolumn{1}{r}{\cellcolor[HTML]{FFFFFF}43.50}                        & \multicolumn{1}{r}{\cellcolor[HTML]{FFFFFF}38.57}                                 & \multicolumn{1}{r|}{\cellcolor[HTML]{FFFFFF}43.50}                                 & \cellcolor[HTML]{FFFFFF}43.50          \\\cline{2-9} 
                              & \cellcolor[HTML]{FFFFFF}TBN                                          & \cellcolor[HTML]{FFFFFF}20.35                     & \cellcolor[HTML]{FFFFFF}20.82                              & \cellcolor[HTML]{FFFFFF}23.06                              & \multicolumn{1}{r}{\cellcolor[HTML]{FFFFFF}31.62}                        & \multicolumn{1}{r}{\cellcolor[HTML]{FFFFFF}38.57}                                 & \multicolumn{1}{r|}{\cellcolor[HTML]{FFFFFF}89.83}                                 & \cellcolor[HTML]{FFFFFF}37.38          \\

                              & TEMA                                          & \cellcolor[HTML]{FFFFFF}20.35                     & \cellcolor[HTML]{FFFFFF}20.82                              & \cellcolor[HTML]{FFFFFF}\textbf{20.69}                              & \multicolumn{1}{r}{\cellcolor[HTML]{FFFFFF}{\color[HTML]{000000} \textbf{21.26}}} & \multicolumn{1}{r}{\cellcolor[HTML]{FFFFFF}{\color[HTML]{000000} 21.61}}          & \multicolumn{1}{r|}{\cellcolor[HTML]{FFFFFF}{\color[HTML]{000000} 24.95}}          & \cellcolor[HTML]{FFFFFF}21.61          \\
\multirow{-4}{*}{1}  & \cellcolor[HTML]{EFEFEF}Ours               & \cellcolor[HTML]{EFEFEF}\textbf{20.20}            & \multicolumn{1}{l}{\cellcolor[HTML]{EFEFEF}\textbf{20.57}} & \multicolumn{1}{l}{\cellcolor[HTML]{EFEFEF}{20.74}} & \cellcolor[HTML]{EFEFEF}{\color[HTML]{000000}{21.45}}            & \multicolumn{1}{r}{\cellcolor[HTML]{EFEFEF}{\color[HTML]{000000} \textbf{20.91}}} & \multicolumn{1}{l|}{\cellcolor[HTML]{EFEFEF}{\color[HTML]{000000} \textbf{21.05}}} & \cellcolor[HTML]{EFEFEF}\textbf{20.82} \\ \hline\rule{0pt}{8pt}
& \cellcolor[HTML]{FFFFFF}Baseline                                         & \cellcolor[HTML]{FFFFFF}46.45                     & \cellcolor[HTML]{FFFFFF}46.45                              & \cellcolor[HTML]{FFFFFF}46.45                                & \multicolumn{1}{r}{\cellcolor[HTML]{FFFFFF}46.45  }                        & \multicolumn{1}{r}{\cellcolor[HTML]{FFFFFF}46.45  }                                 & \multicolumn{1}{r|}{\cellcolor[HTML]{FFFFFF}46.45  }                                 & \cellcolor[HTML]{FFFFFF}46.45          \\\cline{2-9} 
                              & \cellcolor[HTML]{FFFFFF}TBN                                           & \cellcolor[HTML]{FFFFFF}35.50                     & \multicolumn{1}{l}{\cellcolor[HTML]{FFFFFF}36.29}          & \multicolumn{1}{l}{\cellcolor[HTML]{FFFFFF}39.67}          & \cellcolor[HTML]{FFFFFF}{\color[HTML]{000000} 52.73}                     & \cellcolor[HTML]{FFFFFF}{\color[HTML]{000000} 73.24}                              & \multicolumn{1}{r|}{\cellcolor[HTML]{FFFFFF}{\color[HTML]{000000} 98.91}}          & \cellcolor[HTML]{FFFFFF}56.06          \\
                              & TEMA                                          & \cellcolor[HTML]{FFFFFF}35.50                     & \multicolumn{1}{l}{\cellcolor[HTML]{FFFFFF}\textbf{35.86}}          & \multicolumn{1}{l}{\cellcolor[HTML]{FFFFFF}35.83}          & \cellcolor[HTML]{FFFFFF}{\color[HTML]{000000} 37.71}                     & \cellcolor[HTML]{FFFFFF}{\color[HTML]{000000} 41.16}                              & \multicolumn{1}{r|}{\cellcolor[HTML]{FFFFFF}{\color[HTML]{000000} 52.75}}          & \cellcolor[HTML]{FFFFFF}39.80          \\
\multirow{-4}{*}{{2}} & \cellcolor[HTML]{EFEFEF}Ours               & \cellcolor[HTML]{EFEFEF}\textbf{34.63}            & \multicolumn{1}{l}{\cellcolor[HTML]{EFEFEF}{36.11}} & \multicolumn{1}{l}{\cellcolor[HTML]{EFEFEF}\textbf{35.31}} & \cellcolor[HTML]{EFEFEF}{\color[HTML]{000000} \textbf{36.02}}            & \multicolumn{1}{r}{\cellcolor[HTML]{EFEFEF}{\color[HTML]{000000} \textbf{36.32}}} & \multicolumn{1}{l|}{\cellcolor[HTML]{EFEFEF}{\color[HTML]{000000} \textbf{39.30}}}  & \cellcolor[HTML]{EFEFEF}\textbf{36.28} \\ \hline\rule{0pt}{8pt}
& \cellcolor[HTML]{FFFFFF}Baseline                                         & \cellcolor[HTML]{FFFFFF}82.00                     & \cellcolor[HTML]{FFFFFF}82.00                                & \cellcolor[HTML]{FFFFFF}82.00                                & \multicolumn{1}{r}{\cellcolor[HTML]{FFFFFF}82.00  }                        & \multicolumn{1}{r}{\cellcolor[HTML]{FFFFFF}82.00  }                                 & \multicolumn{1}{r|}{\cellcolor[HTML]{FFFFFF}82.00  }                                 & \cellcolor[HTML]{FFFFFF}82.00           \\\cline{2-9} 
                              & \cellcolor[HTML]{FFFFFF}TBN                                           & \cellcolor[HTML]{FFFFFF}66.40                     & \cellcolor[HTML]{FFFFFF}68.60                              & \cellcolor[HTML]{FFFFFF}70.79                              & \cellcolor[HTML]{FFFFFF}83.10                                            & \cellcolor[HTML]{FFFFFF}92.74                                                       & \multicolumn{1}{l|}{\cellcolor[HTML]{FFFFFF}99.22}                                 & \cellcolor[HTML]{FFFFFF}82.89          \\
                              & TEMA                                          & \cellcolor[HTML]{FFFFFF}66.40                     & \cellcolor[HTML]{FFFFFF}66.22                              & \cellcolor[HTML]{FFFFFF}66.64                              & \cellcolor[HTML]{FFFFFF}69.79                                            & \cellcolor[HTML]{FFFFFF}74.94                                                     & \multicolumn{1}{l|}{\cellcolor[HTML]{FFFFFF}84.74}                                 & \cellcolor[HTML]{FFFFFF}72.47          \\
\multirow{-4}{*}{3}  & \cellcolor[HTML]{EFEFEF}Ours               & \cellcolor[HTML]{EFEFEF}\textbf{64.05}            & \multicolumn{1}{l}{\cellcolor[HTML]{EFEFEF}\textbf{64.15}} & \multicolumn{1}{l}{\cellcolor[HTML]{EFEFEF}\textbf{64.90}} & \multicolumn{1}{r}{\cellcolor[HTML]{EFEFEF}\textbf{64.87}}               & \cellcolor[HTML]{EFEFEF}\textbf{66.01}                                            & \multicolumn{1}{l|}{\cellcolor[HTML]{EFEFEF}\textbf{68.79}}                        & \cellcolor[HTML]{EFEFEF}\textbf{65.46} \\ \hline
\end{tabular}
\end{small}
\caption{Ablation study on three corruption benchmarks (1 $\rightarrow$ CIFAR-10-C, 2 $\rightarrow$ CIFAR-100-C, 3 $\rightarrow$ ImageNet-C). }
\label{tab:abla}
\end{table}

\subsection{Ablation Study}\label{sec:abla}
We conduct an ablation study on the importance of the proposed method  under the continual setting. The results are shown in Table~\ref{tab:abla}. 
The Baseline model displays consistent error rates across various test batch sizes. However, the introduction of target normalization statistics (TBN/TEMA) to the model significantly reduced the error rates across all test batch sizes. While with TEMA, an optimal selection of the momentum for target statistics, this component yields robust improvements over all batch sizes, reaching the lowest error rate. Finally, the addition of Layer-wise Rectification Strategy (Ours) leads to the best performance across all batch sizes, showcasing the benefit of a layer-wise normalization rectification strategy for better generalization ability. 

\section{Conclusions and Limitations}
This paper 
examines
test-time degradation by unraveling batch normalization, 
identifying the reduced class diversity in batches as the key issue.
To mitigate this problem and promote test-time performance, 
we 1) introduce TEMA with adaptive data scope, designed to bridge the class diversity gap between training and testing, and 2) propose a novel layer-wise rectification strategy, calibrated using inter-domain divergence to harmonize source and target statistics.
Experiments in diverse real-world scenarios demonstrate the superiority of our method in comparison with the state-of-the-arts. While our proposed strategy shows effectiveness in most scenarios, it presumes samples to be i.i.d, which may not apply to all circumstances. Future work could involve extending the exploration to non-i.i.d settings.

\section{Appendix}

\subsection{Proof of Proposition 1}
\begin{proof}
We begin with the premise of an i.i.d. sample space with $K$ equally probable categories.
For any given $k < K$, the combinations $\mathbf{C}_{K}^{k}$ represent the ways to select $k$ categories from $K$. After choosing $k$ classes, the arrangements where each category gets at least one sample in a batch of $N$ samples are counted by $\mathbf{C}_{N-1}^{k-1}$,
and the total number of ways to assign $N$ samples into $1, 2, \ldots , K$ classes is $\mathbf{C}_{N+K-1}^{K-1}$. As such, the probability of having exactly $k$ unique categories in a batch is the ratio of these combinatory figures. Finally, we can derive the expected category diversity 
by summing all unique categories $k$ times their respective probabilities as
$E(M|N) = \sum_{k=1}^{K} \left[ k \times \frac{\mathbf{C}_{N-1}^{k-1}\mathbf{C}_{K}^{k}}{\mathbf{C}_{N+K-1}^{K-1}} \right].$
\end{proof}

\subsection{Flowchart of Layer-wise Rectification Strategy}
We present a chart in Figure~\ref{fig:alphaupdate} to visually illustrate the procedure of Algorithm 1.  

 \begin{figure}[h]
  \centering
\includegraphics[width=0.47\textwidth]{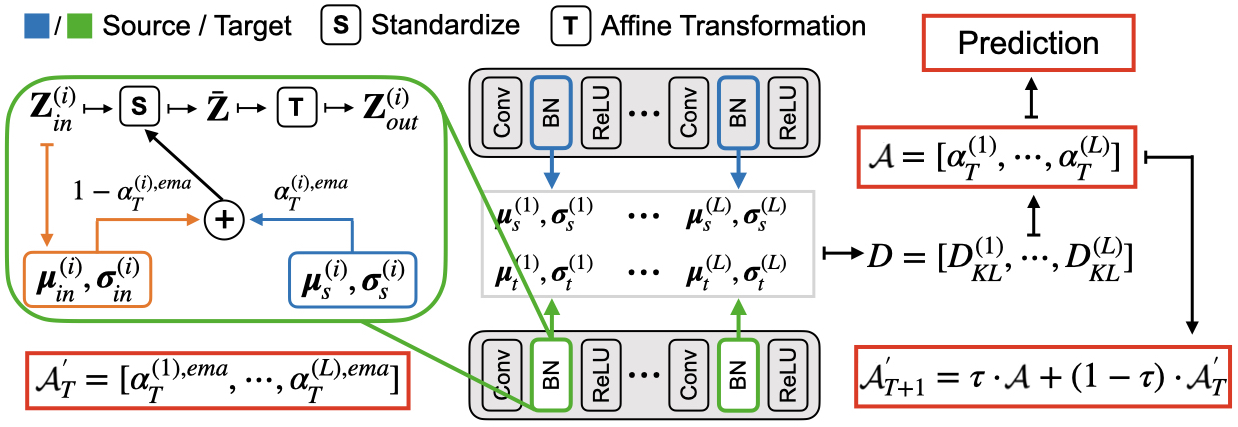}
\caption{Flowchart of Layer-wise Rectification Strategy.}
  \label{fig:alphaupdate}
\end{figure}

\subsection{Model progression with data expansion}
We also examine the relationship between real-time model performance and the volume of incoming data in Figure~\ref{fig:stability} and observe that the adaptation outcome may not be satisfactory during the initial phase. This likely stems from TEMA's early struggles with target distribution estimation, thereby undermining the calibration parameter $\alpha$ and ultimate model performance.
This observation aligns with our stance in the main text that the limited class diversity (in initial TEMA phase) would compromise model performance. Given our primary focus on online Test-Time Adaptation, the issue of total data volume did not initially receive significant attention. We will delve deeper into this problem in future work.
\begin{figure}[h]
  \centering
\includegraphics[width=0.45\textwidth]{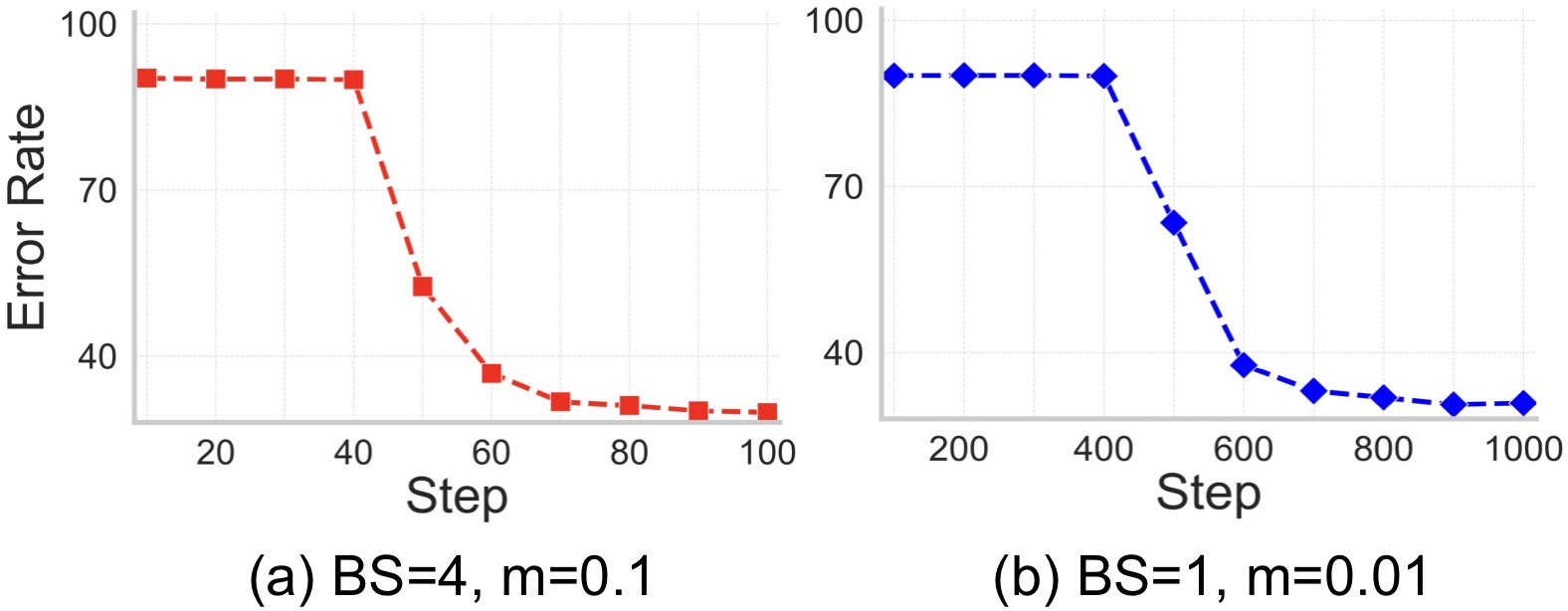}
\caption{Real-time performance on CIFAR-10-C (Gaussian noise). Error rate remains stable for steps after.}
  \label{fig:stability}
\end{figure}

\section*{Acknowledgements} 
This work was supported by: National Natural Science Foundation of China under No. 92370119, No. 62376113,No. 62276258, and No. 62206225; Jiangsu Science and Technology Program (Natural Science Foundation of Jiangsu Province) under No.  BE2020006-4;
Natural Science Foundation of the Jiangsu Higher Education Institutions of China under No. 22KJB520039.

\bibliography{aaai24}

\end{document}